\documentclass{article} 
\usepackage{nips15submit_e,times}

\usepackage{booktabs} 
\usepackage{algorithm}
\usepackage[noend]{algpseudocode}
\usepackage{enumitem}

\usepackage{graphicx}

\usepackage{amsmath}
\usepackage{amssymb}
\usepackage{bbold}
\usepackage{dsfont}
\usepackage{url}

\usepackage{multirow}
\usepackage{caption}
\usepackage{subcaption}
\usepackage{footmisc}
\usepackage{amsthm}
\usepackage{capt-of}
\usepackage[numbers]{natbib}

\usepackage{soul}

\usepackage[toc,page]{appendix}

\newcommand{\xpt}{\edef\f@size{\@xpt}\rm}

\nipsfinaltrue

\def\etc{\emph{etc}}

\renewcommand\vec[1]{\ensuremath\boldsymbol{#1}}
\renewcommand\cdots{...}

\newcommand{\vy}{\boldsymbol{y}}

\newcommand{\mW}{\boldsymbol{W}}
\newcommand{\vx}{\boldsymbol{x}}

\newcommand{\mbr}[1]{\mathbb{R}^{#1}}

\newcommand{\vv}{\boldsymbol{v}}

\newcommand{\idx}[1]{\mathcal{I}_{#1}}

\newcommand{\vz}{\boldsymbol{z}}
\newcommand{\vzeta}{\boldsymbol{\zeta}}

\newcommand{\vphi}{\boldsymbol{\phi}}

\newcommand{\enorm}[1]{\left\|{#1}\right\|_2}

\newcommand{\set}[1]{\left\{#1\right\}}

\DeclareMathOperator*{\argmin}{arg\,min}

\newtheorem{proposition}{Proposition}

\def\eg{\emph{e.g.}}

\newcommand{\overbartwo}[1]{\mkern 2mu\overline{\mkern-2mu#1\mkern-2mu}\mkern 2mu}

\newcommand{\vvarphi}{\boldsymbol{\varphi}}

\newcommand{\cov}{\boldsymbol{\Sigma}}

\newcommand{\mPhi}{\boldsymbol{\Phi}}
\newcommand{\mLam}{\boldsymbol{\Lambda}}

\newcommand{\vmu}{\boldsymbol{\mu}}

\newcommand{\mP}{\boldsymbol{\Theta}}

\newcommand{\stkout}[1]{{\ifmmode\text{\sout{\ensuremath{#1}}}\else\sout{#1}\fi}}

\newcommand{\comment}[1]{}

\title{\Large CNN-based Action Recognition and Supervised Domain Adaptation on 3D Body Skeletons via Kernel Feature Maps}

\author{Yusuf Tas\thanks{
This work is under review. Please respect the authors' efforts by not copying/borrowing/plagiarizing bits and pieces of this work for your own gain.}\textsuperscript{$\;\,$,1,2}\qquad Piotr Koniusz\textsuperscript{$*$,1,2}\\
$^1$Data61/CSIRO, $^2$Australian National University\\
firstname.lastname@\{data61.csiro.au\textsuperscript{1}, anu.edu.au\textsuperscript{2}\}
}

\newcommand\keywords[1]{}
\nipsfinalcopy 

\begin{document}

\maketitle

\def\arxiv{arxiv}
\begin{abstract}
Deep learning is ubiquitous across many areas areas of computer vision. It often requires large scale datasets for training before being fine-tuned on small-to-medium scale problems.  Activity, or, in other words, action recognition, is one of many application areas of deep learning. While there exist many Convolutional Neural Network architectures that work with the RGB and optical flow frames, training on the time sequences of 3D body skeleton joints is often performed via recurrent networks such as LSTM.

In this paper, we propose a new 
representation which encodes sequences of 3D body skeleton joints in texture-like representations derived from mathematically rigorous kernel methods. Such a representation becomes the first layer in a standard CNN network \eg, ResNet-50, which is then used in the supervised domain adaptation pipeline to transfer information from the source to target dataset. This lets us leverage the available Kinect-based data beyond training on a single dataset and 
outperform simple fine-tuning on any two datasets combined in a naive manner. More specifically, in this paper we  utilize the overlapping classes between datasets. We associate datapoints of the same class via so-called commonality, known from the supervised domain adaptation. 
We demonstrate state-of-the-art results on three publicly available benchmarks.
\end{abstract}
\section{Introduction}
In recent years, we have witnessed a great increase in the usage and development of deep learning frameworks such as Convolutional Neural Networks (CNN). 
Starting from an outstanding paper on the AlexNet architecture \cite{alex2012}, application areas such as text processing, speech recognition, feature learning and extraction, semantic segmentation, object detection and recognition have adopted deep learning since \cite{girshick2014,collobert2008,hinton2012deep,ren2015faster,donahue2014decaf}. 

Action recognition aims to distinguish between different action classes such as walking, pushing, hand shaking, kicking, punching, to name but a few of action concepts. The ability to recognize human actions enables progress in many application areas verging from the video surveillance to human-computer interaction \cite{herath2017going}. Videos have been the main source of the data for action recognition, however, data sources such as RGB-D have become popular since the introduction of the Kinect sensor as they facilitate tracking 3D coordinates of human skeleton body joints which form time sequences.
Similar to the object classification, the past action recognition systems relied on handcrafted spatio-temporal feature descriptors such as \cite{bobick2001recognition,laptev2005space,klaser2008spatio}, with a notable shift to deep learning frameworks \cite{ji20133d,karpathy2014large,simonyan2014two,spat_tmp_res} which combine RGB and optical flow CNN streams. However, little has been done to investigate the use of sequences of 3D body skeleton joints in CNNs, with an exception of \cite{new_skeleton}.

In this paper, we focus on the action recognition of sequences of 3D body skeleton joints and propose an 
input layer which we combine with off-the-shelf CNNs. This enables us to further pursue our goal of the supervised domain adaptation  to leverage Kinect-based datasets as the known supervised domain adaptation approaches \cite{tzeng_transfer,me_domain} are based on CNNs rather than the recurrent networks such as RNN and LSTM \cite{du2015hierarchical,zhu2016co,liu2016spatio}.

It has been shown that in deep networks, early layers recognize edges, corners, basic shapes and structures; prompting similarity to handcrafted features. However, in the consecutive layers, learned filters respond to more complex stimuli \cite{zeiler2014visualizing}. This attractive property of deep learning together with the shift-invariance of pooling result in a superior performance compared to handcrafted features. 
Even more powerful are the residual CNN representations \cite{he_residual,spat_tmp_res} which have the ability to bypass the local minima resulting from the non-convex nature of CNN networks. Therefore, our work is based on the ResNet-50 model.

Papers on human action recognition use several datasets such as KTH\cite{schuldt2004recognizing}, HMDB-51\cite{kuehne2011hmdb}, SBU-Kinect-Interaction\cite{yun2012two}, UTKinect-Action3D\cite{xia2012view}, NTU RGB+D\cite{shahroudy2016ntu}, most of which have a significant overlap of the class concepts describing actions. 
Thus, we adopt a domain adaptation approach based on the class-wise mixture of alignments of second-order scatter matrices \cite{me_domain}. We apply it to time sequences of 3D body skeleton joints to transfer the knowledge between the overlapping classes of two datasets. 
Our contributions are:

\setlist[enumerate,1]{label={(\roman*)}}
\vspace{0.1cm}
\begin{enumerate}[topsep=0pt,itemsep=0.1pt,leftmargin=16pt] 
\item We propose a novel method that encodes sequences of 3D body skeleton joints into a kernel feature map representation suitable for the use with off-the-shelf CNNs. Our representation enjoys a sound mathematical derivation based on kernel methods \cite{scholkopf_kern}.
\item We are the first to adapt the supervised domain adaptation \cite{me_domain} for the action recognition on time sequences of 3D body skeleton joints. We extend the so-called mixture alignment of classes \cite{me_domain} to work with datasets which class concepts match partially.  
\end{enumerate}


\section{Related Work}
First, we describe the most popular CNN action recognition models followed by the 3D body joint representations. Subsequently, we focus on the most related to our approach techniques.

\vspace{0.05cm}
\noindent{\bf{CNNs for Action Recornition.}} 
Ji et al. \cite{ji20133d} propose a CNN model to utilize 3D structure in videos by multiple convolution operations. Karpathy et al. \cite{karpathy2014large} propose a method called `slow fusion' which learns temporal information by feeding sequentially parts from the video to the algorithm. Simonyan and Zisserman \cite{simonyan2014two} propose a two-stream network which benefits from both spatial domain with RGB images and temporal domain with optical flow. 

\vspace{0.05cm}
\noindent{\bf{3D Body Joint Sequences.}} 
Systems such as Microsoft Kinect can locate body parts and produce a set of articulated connected body joints that evolve in time and form time sequences of 3D coordinates \cite{zatsiorsky_body}. 
Action recognition via sequences of 3D body skeleton joints has received a wider attention in the community, as witnessed by a survey paper \cite{presti20153d}.

While the RGB-based video sequences contain background, clutter and other sources of noise, the advantage of skeleton-based representations is that they can accurately describe human motion. This was first demonstrated by Johansson \cite{johansson_lights} in his seminal experiment involving the moving lights display. By observing moving body joints that represent \eg, elbow, wrist, knee, ankle, one can tell the action taking place. Moreover, sensors such as Kinect fuse depth and RGB frames, and combine the body joint detector, tracker \cite{shotton2013real}, and segmentation to robustly separate the background clutter from the subject's motion. 
For any given subject/action, the 3D positions of body joints evolve spatio-temporally. 

Various descriptors of body joints have been proposed \eg, the motion of 3D points is used in \cite{hussein_action,lv_3daction}, orientations w.r.t. a reference axis are used by \cite{parameswaran_viewinvariance} and relative body-joint positions are used in \cite{wu_actionlets,yang_eigenjoints}. Connections between body segments are used in \cite{yacoob_activities,ohn_hog2,ofli_infjoints,vemulapalli_SE3}. In contrast, we represent sequences of 3D body-joints by a kernel whose linearization yields texture-like feature maps which capture complex statistics of joints for CNN.

\vspace{0.05cm}
\noindent{\bf{Map generation from 3D Body Joint Sequences.}} 
%
A recent paper \cite{new_skeleton} forms texture arrays from 3D coordinates of body joints. Firstly, 4 key body joints are chosen as reference to form a center of coordinate system by which the 3D positions of remaining body joints are shifted before conversion into cylindircal coordinates. Coordinate of each body joint is stacked along rows while temporal changes happen along columns. This results in 12 maps resized to $224\!\times\!224$ and passed to 12 CNN streams combined at the {\em FC} layer.

Our method is somewhat related in that our feature maps resemble textures. However, our maps are obtained by a linearization of the proposed by us kernel function which measures similarity between any pair of two sequences. 
The parameters of these kernels introduce a desired degree of shift-invariance in both spatial and temporal domains. 
Our approach is also somewhat related to kernel descriptors for image recognition \cite{ker_des}, 
 Convolutional Kernel Networks \cite{ckn} and 
kernelized covariances \cite{cavazza_kercov} for action recognition, a time series kernel on scatter matrices 
\cite{gaidon_timekern} and a spatial compatibility kernel \cite{tensor_eccv} that yields a tensor descriptors. In contrast, our layer captures third-order co-occurrences between 3D skeleton body joints and temporal domain to produce texture-like feature maps that are passed to CNN.

\begin{figure}[t]
\centering
%
\begin{subfigure}[b]{0.66\linewidth}
\centering\includegraphics[trim=0 0 0 0, clip=true, height=2.2cm]{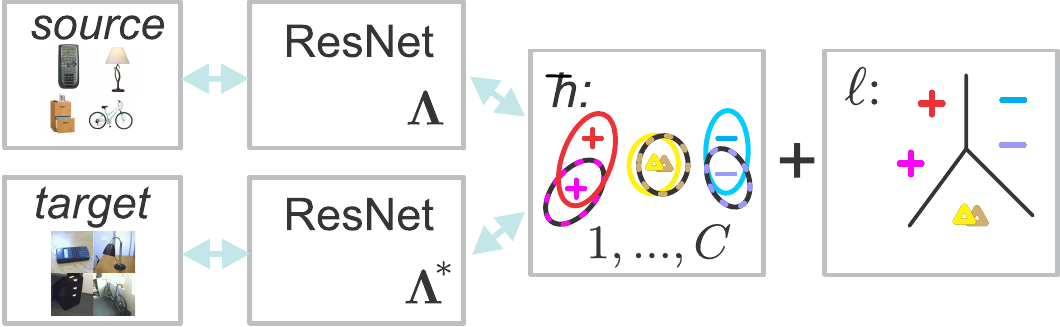}
\caption{\label{fig:cnn1}}
\vspace{-0.1cm}
\end{subfigure}
\begin{subfigure}[b]{0.33\linewidth}
\centering\includegraphics[trim=0 0 0 0, clip=true, height=2.2cm]{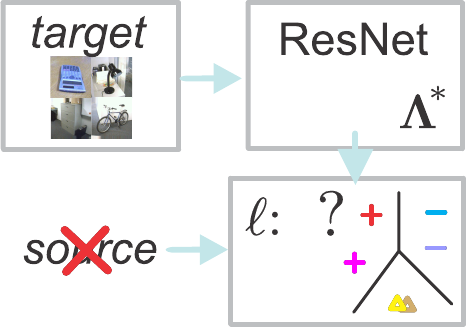}
\caption{\label{fig:cnn2}}
\vspace{-0.1cm}
\end{subfigure}
\ifdefined\arxiv
\vspace{-0.45cm}
\else
\vspace{0.2cm}
\fi
\caption{Supervised Domain Adaptation \cite{me_domain}. Figure \ref{fig:cnn1}: The source and target network streams are combined by the classification and alignment losses $\ell$ and $\hbar$ (end-to-end learning) which operate on the feature vectors from the final {\em FC} layers of ResNet-50 streams $\mLam$ and $\mLam^{*\!}$. Loss $\hbar$ aligns covariances for $C$ classes to facilitate transfer learning. Figure \ref{fig:cnn2}: At the test time, the target stream only and the classifier are used.
}\vspace{-0.3cm}
\label{fig:cnn_all}
\end{figure}

%
\vspace{0.05cm}
\noindent{\bf{Supervised Domain Adaptation.}} 
In this paper, we employ the supervised domain adaptation which role is to transfer knowledge from the labeled source to labeled target dataset and outperform naive fine-tuning on combined datasets. We adapt an approach \cite{me_domain} based on the mixture of alignments of second-order statistics. One alignment per class per source and target streams is performed  
to discover the so-called commonality \cite{me_domain} between the data streams. 
Thus, both CNN streams learn a transformation of the data into this shared commonality. Figures \ref{fig:cnn1} and \ref{fig:cnn2} show the training and testing procedures. Training requires a trade-off between alignment and training losses $\hbar$ and $\ell$ operating on source and target streams $\mLam$ and $\mLam^{*\!}$. Testing uses only the target stream $\mLam^{*\!}$ and the pre-trained classifier.

Approach \cite{me_domain} assumes that the source and target data have to share the same set of labels. We relax this assumption to perform transfer between the classes shared between both datasets. Thus, we employ separate source and target classifiers and perform the alignment. 

\section{Preliminaries}
\label{sec:background}

In what follows, we explain our notations and the necessary background on shift-invariant RBF kernels and
their linearization, which are needed for deriving a kernel on sequences on 3D body skeleton joints together with its linerization into feature maps.

\vspace{0.05cm}
\noindent{\bf{Notations.}} 
\label{sec:notations}
%
%
The Kronecker product is denoted by $\otimes$. $\idx{N}$ denotes the index set $\set{1, 2,\cdots,N}$. 
%
%
We use the MATLAB notation $\vv\!=\![\text{begin}\!:\!\text{step}\!:\!\text{end}]$ to generate a  vector $\vv$ with elements starting as {\em begin}, ending as {\em end}, with stepping equal {\em step}. Operator `$;$' in $[\vx; \vy]$ concatenates vectors $\vx$ and $\vy$ (or scalars) while $[\mPhi_i]_{i\in\idx{J}}$ concatenates $\mPhi_1,\cdots,\mPhi_J$ along rows.

\vspace{0.05cm}
\noindent{\bf{Kernel Linearization.}} 
\label{sec:kernel_linearization}
In the sequel, we use Gaussian kernel feature maps 
detailed below to embed 3D coordinates and their corresponding temporal time stamp into a non-linear Hilbert space and perform linearization which will result in our texture-like feature maps.

\begin{proposition}
\label{pr:gaus_lin}
Let $G_{\sigma}(\vx\!-\!\vy)=\exp(-\!\enorm{\vx\!-\!\vy}^2/{2\sigma^2})$ denote a Gaussian RBF kernel centered at $\vy$ and having a bandwidth $\sigma$. Kernel linearization refers to rewriting this $G_{\sigma}$ as an inner-product of two infinite-dimensional feature maps. To obtain these maps, we use a fast approximation method based on probability product kernels \cite{jebara_prodkers}.
Specifically, we employ the inner product of $d'$-dimensional isotropic Gaussians given $\vx,\vy\!\in\!\mbr{d'}\!$. Consider equation: 
\begin{align}
&\!\!\!\!\!\!\!G_{\sigma}\!\left(\vx\!-\!\vy\right)\!\!=\!\!\left(\frac{2}{\pi\sigma^2}\right)^{\!\!\frac{d'}{2}}\!\!\!\!\!\!\int\limits_{\vzeta\in\mbr{d'}}\!\!\!\!G_{\sigma/\sqrt{2}}\!\!\left(\vx\!-\!\vzeta\right)G_{\sigma/\sqrt{2}}(\vy\!\!-\!\vzeta)\,\mathrm{d}\vzeta.
\label{eq:gauss_lin}
\end{align}
Eq. \eqref{eq:gauss_lin} can be approximated by replacing the integral with the sum over $Z$ pivots $\vzeta_1,\cdots,\vzeta_Z$: 
\begin{align}
&\!\!\!\!\!\!\!G_{\sigma}(\vx\!-\!\vy)\approx\left<\sqrt{c}\vvarphi(\vx), \sqrt{c}\vvarphi(\vy)\right>, \text{ where } \vvarphi(\vx)=\left[{G}_{\sigma/\sqrt{2}}(\vx-\vzeta_1),\cdots,{G}_{\sigma/\sqrt{2}}(\vx-\vzeta_Z)\right]^T\!\!,\label{eq:gauss_lin2}
\end{align}
and $c$ represents a constant (it impacts the overall magnitude only so we set $c\!=\!1$). We refer to \eqref{eq:gauss_lin2} (left) as the linearization of the RBF kernel and \eqref{eq:gauss_lin2} (right) as an RBF feature map{\color{red}\footnotemark[1]}. 
\end{proposition}
\begin{proof}
Rewrite the Gaussian kernel as the probability product kernel \cite{jebara_prodkers} (Sec. 3.1).
\end{proof}
\footnotetext[1]{\label{foot:maps}Note that (kernel) feature maps are not conv. CNN maps. They are two separate notions that  share the name.}

\section{Proposed Method}

Below, we formulate the problem of action recognition from sequences of 3D body skeleton joints, followed by our kernel formulation capturing actions, and its linearization into feature maps which we further feed to off-the-shelf CNN for classification.

\begin{figure}[t]
\centering
%
\centering\includegraphics[trim=0 0 0 0, clip=true, width=8.2cm]{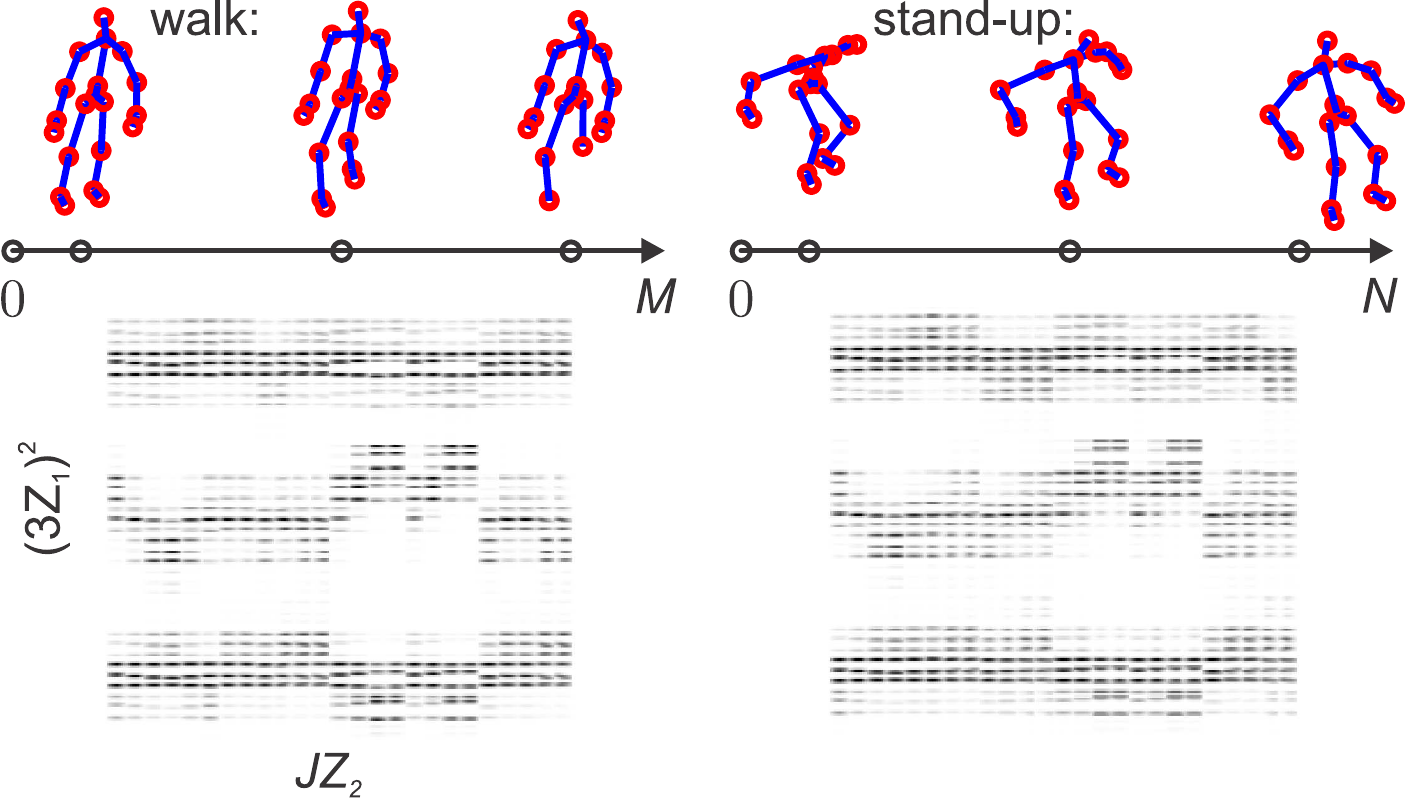}
%
%
\ifdefined\arxiv
\vspace{-0.0cm}
\else
\vspace{0.2cm}
\fi
\caption{Visualization of the feature maps of sequences of 3D body skeleton joints. Note that irrespectively of the sequence length, we always obtain $\mPhi\!\in\!\mbr{225\times 225}$ feature maps.
}\vspace{-0.3cm}
\label{fig:maps1}
\end{figure}

\subsection{Generation of Feature Maps via Kernel Linearization}
\label{sec:feat_maps}

Let dataset consist of sequences of $J$ 3D body skeleton joints describing human pose skeleton evolving in time. For brevity, we assume each sequence consists of $M$ frames. However, our formulation is applicable to sequences of variable lengths \eg,~$M$ and $N$. Our pose sequence $\Pi$ is defined as:
\begin{equation}
\Pi = \set{\vx_{is}\in\mbr{3},i\in\idx{J}, s\in\idx{M}}.
\end{equation}
Each sequence $\Pi$ is described by one of $C$ action labels. We use the sequence $\Pi$ to generate a feature map which can be considered a descriptor of action associated with  $\Pi$. Then, such feature maps are generated from given datasets and then fed to the source and target CNN streams with the goal of performing the supervised domain adaptation. Figure \ref{fig:maps1} illustrates the sequences and feature maps obtained as a result of the process detailed next.

In what follows, we want to measure the similarity between any two action sequences in terms of their 3D body skeleton joints as well as their evolution in time. We normalize each skeleton w.r.t. the chest joint (chosen to be the center). Moreover, we normalize such relative coordinates by their total variance computed over the training data.  Let $\Pi_A$ and $\Pi_B$ be two sequences, each with $J$ joints, and $M$ and $N$ frames, respectively. Further, let $\vx_{is}\!\in\!\mbr{3}$ and $\vy_{jt}\!\in\!\mbr{3}$ correspond to coordinates of joints of body skeletons of $\Pi_A$ and $\Pi_B$, respectively. We define our~\emph{sequence kernel} (SCK) between $\Pi_A$ and $\Pi_B$ as:
\begin{align}
& K(\Pi_A,\Pi_B) = \frac{1}{MN}\!\!\! \sum\limits_{i\in\idx{J}} \sum\limits_{s\in\idx{M}}\sum\limits_{t\in\idx{N}}\!K_{\sigma_1}\!\left(\vx_{is} - \vy_{it}\right)^2 G_{\sigma_2}(\frac{s}{M}-\frac{t}{N}),\label{eq:ker1}
\end{align}
where $1/(MN)$ is a normalization constant, and $G_{\sigma_1}$ and $G_{\sigma_2}$ are subkernels that capture the similarity between the 3D body skeleton joints and temporal alignment, respectively. Therefore, we have two parameters $\sigma_1$ and $\sigma_2$ which control the level of tolerated invariance w.r.t. misalignment of 3D body joints and their temporal positions in two sequences, respectively. Moreover, the square of $K_{\sigma_1}$ in Eq. \eqref{eq:ker1} captures co-occurrences of x, y, and z Cartesian coordinates of each 3D body joint--it is shown below that the square operation corresponds to the Kronecker product which is known to capture co-occurrences.

First, we define $K_{\sigma_1}(\vx-\vy)\!=\!\sum_{i\in\idx{3}}G_{\sigma_1}(x^i-y^i)$ where superscript $i$ chooses x-, y-, or z-axis of a 3D coordinate vector. Next, we linearize the above kernel using the theory from Section \ref{sec:kernel_linearization} so that $K_{\sigma_1}(\vx-\vy)\!\approx\!\sum_{i\in\idx{3}}\phi(x^i)^T\phi(y^i)$, which gives the dot-product of concatenations $K_{\sigma_1}(\vx-\vy)\!\approx\![\phi(x^1); \phi(x^2); \phi(x^3)]^T[\phi(y^1); \phi(y^2); \phi(y^3)]$. In what follows, we write for simplicity that $K_{\sigma_1}(\vx-\vy)\!\approx\!\phi(\vx)^T\phi(\vy)$. Moreover, temporal kernel $G_{\sigma_2}(\frac{s}{M}-\frac{t}{N})\approx\vz(s/M)^T\vz(t/N)$. The above linearizations combined with Eq. \eqref{eq:ker1} lead to:
\begin{align}
&K(\Pi_A,\Pi_B)\!\approx\!\frac{1}{MN}\!\!\sum\limits_{i\in\idx{J}}\sum\limits_{s\in\idx{M}}\!\sum\limits_{t\in\idx{N}}\!\!\!\big(
\vphi(\vx_{is})^T\vphi(\vy_{it})\big)^2\vz(s/M)^T\vz(t/N),
\label{eq:ker1b}
\end{align}
which can be further rewritten into Eq. \eqref{eq:ker1c} and simplified by Eq. \eqref{eq:ker1d}:
\begin{align}
&\!\!K(\Pi_A,\Pi_B)\!\approx\!\frac{1}{MN}\!\!\sum\limits_{i\in\idx{J}}\sum\limits_{s\in\idx{M}}\!\sum\limits_{t\in\idx{N}}\!\!\!\left<
(\vphi(\vx_{is})\!\otimes\!\vphi(\vx_{is}))\vz(s/M)^T\!\!,\,
(\vphi(\vy_{it})\!\otimes\!\vphi(\vy_{it}))\vz(t/N)^T\right>\label{eq:ker1c}\\
&\!\!\qquad\qquad\;=\sum\limits_{i\in\idx{J}}\!\!\Big<
\!\frac{1}{M}\!\!\sum\limits_{s\in\idx{M}}\!\!(\vphi(\vx_{is})\!\otimes\!\vphi(\vx_{is}))\vz(s/M)^T\!\!,
\frac{1}{N}\!\!\sum\limits_{t\in\idx{N}}\!\!(\vphi(\vy_{it})\!\otimes\!\vphi(\vy_{it}))\vz(t/N)^T\Big> \;\Rightarrow\nonumber\\
%
%
&\!\!K(\Pi_A,\Pi_B)\!\approx\!\left<\mPhi(\Pi_A), \mPhi(\Pi_B)\right>,\;\text{where}\label{eq:ker1d}\\
&\!\!\mPhi(\Pi_A)\!=\!\Big[\frac{1}{M}\!\!\sum\limits_{s\in\idx{M}}\!\!(\vphi(\vx_{is})\!\otimes\!\vphi(\vx_{is}))\vz(s/M)^T\Big]_{i\in\idx{J}}\!, 
\mPhi(\Pi_B)\!=\!\Big[\frac{1}{N}\!\!\sum\limits_{t\in\idx{N}}\!\!(\vphi(\vy_{it})\!\otimes\!\vphi(\vy_{it}))\vz(t/N)^T\Big]_{i\in\idx{J}}\!,\nonumber
\end{align}
and $\mPhi(\Pi)$ is our texture-like feat. map for a chosen sequence $\Pi$.

We choose $Z_1\!=\!5$ pivots pivots $\vzeta\!=\![\zeta_1,\cdots,\zeta_{Z_1}]^T$ for $G_{\sigma_1}$ which are sampled on interval $[-1; 1]$ with equal steps \eg, $\vzeta\!=\![-1:2/(Z_1\!-\!1):1]$. This results in a $3Z_1$ dimensional map that approximates $K_{\sigma_1}$. For $G_{\sigma_2}$, we choose such an integer number of pivots $Z_2$ that $Z_2J\!=\!225$. We sample these pivots on interval $[0; 1]$. This way, we obtain $\mPhi\!\in\!\mbr{Z_1^2\times Z_2J}$ which can be readily fed to an off-the-shelf CNN stream. Figure \ref{fig:maps2} demonstrates the impact of $\sigma_1$ and $\sigma_2$ radii on the feature maps $\mPhi$.
%
%
%
%
%
%
%
\newcommand{\SrcImgWH}{2.0cm}
\begin{figure}[t]
\centering
%
\begin{subfigure}[b]{0.328\linewidth}
\centering
\includegraphics[trim=0 0 0 0, clip=true, width=\SrcImgWH, height=\SrcImgWH]{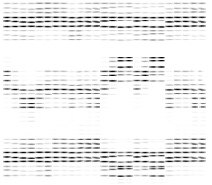}
\includegraphics[trim=55 55 70 23, clip=true, width=\SrcImgWH, height=\SrcImgWH]{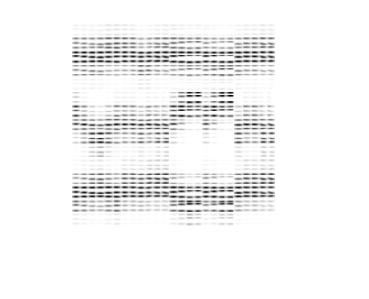}\\
\includegraphics[trim=55 55 70 23, clip=true, width=\SrcImgWH, height=\SrcImgWH]{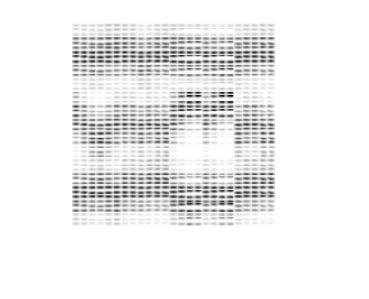}
\includegraphics[trim=55 55 70 23, clip=true, width=\SrcImgWH, height=\SrcImgWH]{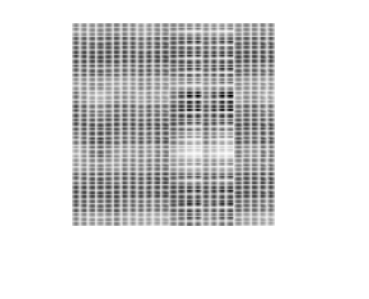}
\vspace{-0.1cm}
\caption{\label{fig:maps2a}}
\end{subfigure}
\begin{subfigure}[b]{0.328\linewidth}
\centering
\includegraphics[trim=55 55 70 23, clip=true, width=\SrcImgWH, height=\SrcImgWH]{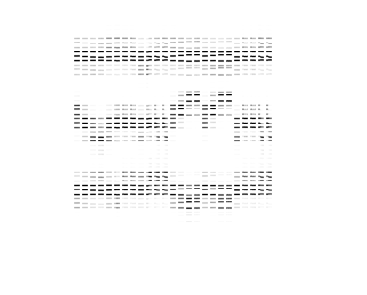}
\includegraphics[trim=55 55 70 23, clip=true, width=\SrcImgWH, height=\SrcImgWH]{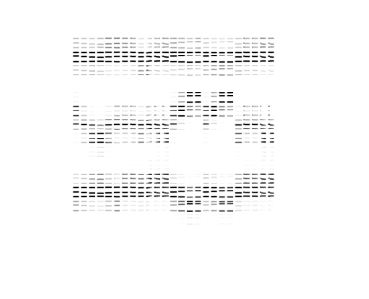}\\
\includegraphics[trim=55 55 70 23, clip=true, width=\SrcImgWH, height=\SrcImgWH]{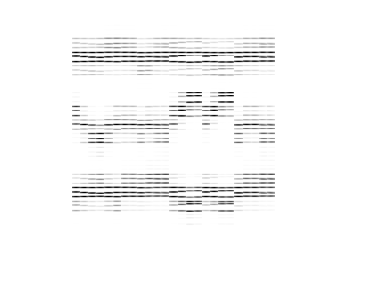}
\includegraphics[trim=55 55 70 23, clip=true, width=\SrcImgWH, height=\SrcImgWH]{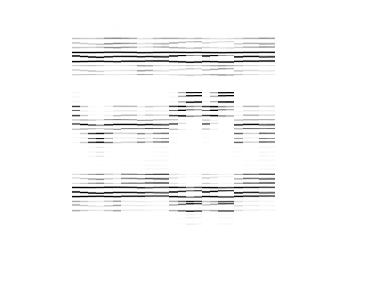}
\vspace{-0.1cm}
\caption{\label{fig:maps2b}}
\end{subfigure}
\begin{subfigure}[b]{0.328\linewidth}
\centering
\includegraphics[trim=0 0 0 0, clip=true, width=\SrcImgWH, height=\SrcImgWH]{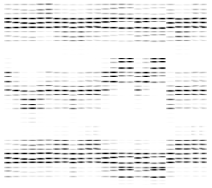}
\includegraphics[trim=0 0 0 0, clip=true, width=\SrcImgWH, height=\SrcImgWH]{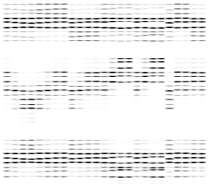}\\
\includegraphics[trim=0 0 0 0, clip=true, width=\SrcImgWH, height=\SrcImgWH]{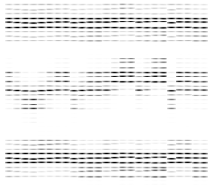}
\includegraphics[trim=0 0 0 0, clip=true, width=\SrcImgWH, height=\SrcImgWH]{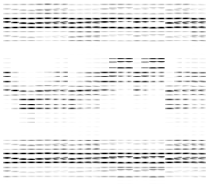}
\vspace{-0.1cm}
\caption{\label{fig:maps2c}}
\end{subfigure}
\ifdefined\arxiv
\vspace{-0.3cm}
\else
\vspace{0.2cm}
\fi
\caption{Illustration of the impact of $\sigma_1\!=\!0.4,0.6,0.8,1.5$ and $\sigma_2\!=\!0.02, 0.1, 1.0, 5.0$ (in the scanline order) on feature maps are given in Figures \ref{fig:maps2a} and \ref{fig:maps2b}, respectively. Figure \ref{fig:maps2c} shows four different maps for four different sequences. Note the subtle differences.
}\vspace{-0.3cm}
\label{fig:maps2}
\end{figure}
\begin{figure*}[b]
\centering
\includegraphics[trim=0 0 0 0, clip=true, width=12cm]{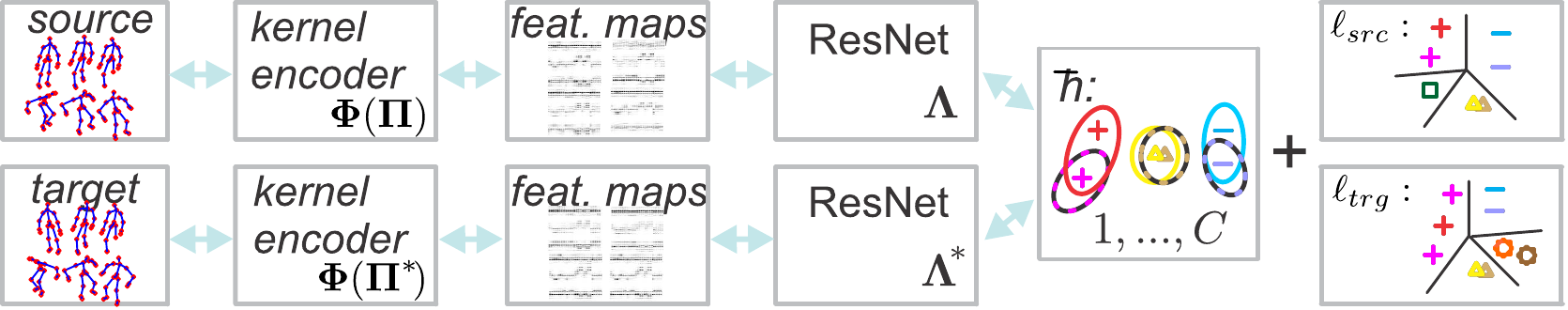}
\vspace{0.2cm}
\caption{Our pipeline: combining the 3D body skeleton encoding and the supervised domain adaptation. Unlike \cite{me_domain}, we utilize two classifiers (one per network stream) and perform alignment between the classes that are shared between the source and target datasets.}\vspace{-0.3cm}
\label{fig:our_pipe}
\end{figure*}
Our feature map is similar in spirit to Convolutional Kernel Networks \cite{ckn} for image classification which demonstrated that the linearization of a carefully designed kernel adheres to standard CNN operations such as convolution, non-linearity and pooling. This motivates our belief that our feature maps are more suited/compatible for interfacing with CNNs than ad-hoc texture-like representations \cite{new_skeleton}.

\subsection{Alignment of Second-order Statistics}

Our final pipeline is illustrated in Figure \ref{fig:our_pipe}. As our ultimate goal is to transfer knowledge between Kinect-based datasets, we combine the described in Section \ref{sec:feat_maps} encoder of sequences of 3D body skeleton joints together with the supervised domain adaptation algorithm {\em So-HoT} \cite{me_domain}. Section {\color{cyan}1} (supplementary material) details this algorithm. 
{\em So-HoT} yields state-of-the-art results on the Office dataset \cite{saenko2010adapting}, however, it works with datasets which are described by the same class concepts. 
Thus, we adapt their algorithm to our particular needs \eg, we only perform the alignment of second-order statistics between the classes that are shared between the source and target datasets. Moreover, we employ separate classifier losses $\ell_{src}$ and $\ell_{trg}$ for the source and target stream, respectively. The separate target classifier allows the target network to work with class labels absent from the source dataset. At the test time, we cut off the source stream (and the source classifier), as illustrated in Figure \ref{fig:cnn2}.

Algorithm {\color{cyan}1} (supplementary material) details how we perform domain adaptation. We enable the alignment loss $\hbar$ only if the source and target batches correspond to the same class. Otherwise, the alignment loss is disabled and the total loss uses only the classification log-losses $\ell_{src}$ and $\ell_{trg}$. To generate the source and target batches that match w.r.t. the class label, we re-order source and target datasets class-by-class and thus each source/target batch contains only one class label at a time. Once all source and target datapoints with matching class labels are processed, remaining datapoints are processed next. Lastly, we refer readers interested in the details of the So-HoT algorithm to paper \cite{me_domain} for specifics of the $\hbar$ loss.

\section{Experiments}
\label{sec:exp}


Below, we detail our network setting, datasets and we show experiments on our feature maps for sequences of 3D body skeleton joints in the context of the supervised domain adaptation. 

\vspace{0.05cm}
\noindent{\bf{Network Model.}} 
%
We use the two streams network architecture from \cite{me_domain}. For each CNN stream, we chose the Residual CNN model ResNet-50 \cite{he_residual} pre-trained on ImageNet dataset \cite{krizhevsky2012imagenet} for both source and target streams. The {\em Pool-5} layers of the source and target streams are forwarded to a fully connected layer {\em FC} with 512 hidden units and this is forwarded to both the classification weight layer and the so-called alignment loss \cite{me_domain}. Two classifiers are used for the source and target streams. Moreover, the alignment loss is activated when the generated source and target mini-batches contain datapoints with the same class labels. See Algorithm {\color{cyan}1} (supplementary material) for more details and Figure \ref{fig:our_pipe} for the network setting.

The training is performed by the Stochastic Gradient Descent (SGD) with the momentum set to 0.9. Mini-batch sizes differ depending on both the source and target dataset. 

\vspace{0.1cm}
\noindent{\bf{Datasets.}} 
%
We use the NTU RGB-D, SBUKinect Interaction and UTKinect-Action3D datasets.

\vspace{0.1cm} 
\noindent{\em NTU RGB-D} \cite{shahroudy2016ntu}, the largest action recognition dataset to date, contains $\sim$56000 sequences of 60 distinct action classes and sequences of actions/interactions performed by 40 different subjects. 3D coordinates of 25 body joints are provided. We use the cross-subject evaluation protocol \cite{shahroudy2016ntu} and used only the train split as our source data. For pre-processing, we translated 3D body joints by the joint-2 (middle of the spine) and we chose the body with the largest 3D motion as the main actor for the multi-actor sequences.




\vspace{0.1cm}
\noindent{\em{SBUKinect}} \cite{yun2012two} contains videos of 8 interaction categories between two people, and 282 skeleton 
 sequences with 15 3D body joints. 
Although the locations of body joints are noisy \cite{yun2012two} and pre-processing is common \cite{zhu2016co}, we do not perform any pre-processing or data augmentation in contrast to \cite{new_skeleton}. 
In domain adaptation setting, we use the NTU training set as the source and SBU as the target data. 
For evaluation, we follow \cite{yun2012two} and use 5-fold cross-validation on the given splits. As each sequence contains 2 persons, we used each skeleton as a separate training datapoint. For testing, we averaged predictions over such pairs. 

\vspace{0.1cm}
\noindent{\em{UTKinect-Action3D}} \cite{xia2012view} contains 10 action captured by Kinect, 
199 sequences, and 20 3D body skeleton joints. 
\begin{table}[t]
\parbox{0.525\textwidth}{
\renewcommand{\arraystretch}{0.8}{
\setlength{\tabcolsep}{0.25em}
\fontsize{9}{9.5}\selectfont
\begin{tabular}{c|c|c}
\hline
Methods&SBU&UTK\\
\hline
Cylindrical textures, $1\!\times$CNN \cite{new_skeleton}&$89.37\%$&$95.0\%$\\
Cylindrical textures, $3\!\times$CNN \cite{new_skeleton}&$90.24\%$&$95.9\%$\\ 
Kernel feature maps, $1\!\times$CNN (ours)&$\bf91.13\%$&$\bf96.5\%$\\ 
\hline
\end{tabular}
}
\vspace{0.2cm}
\caption{Comparisons of texture representations.}
\label{tab:text}
}
\hspace{0.1cm}
%
%
\parbox{0.46\textwidth}{
%
%
\vspace{-0.1cm}
\centering
\begin{subfigure}[b]{0.42\linewidth}
\centering
\includegraphics[trim=0 0 0 0, clip=true, width=2.25cm, height=2cm]{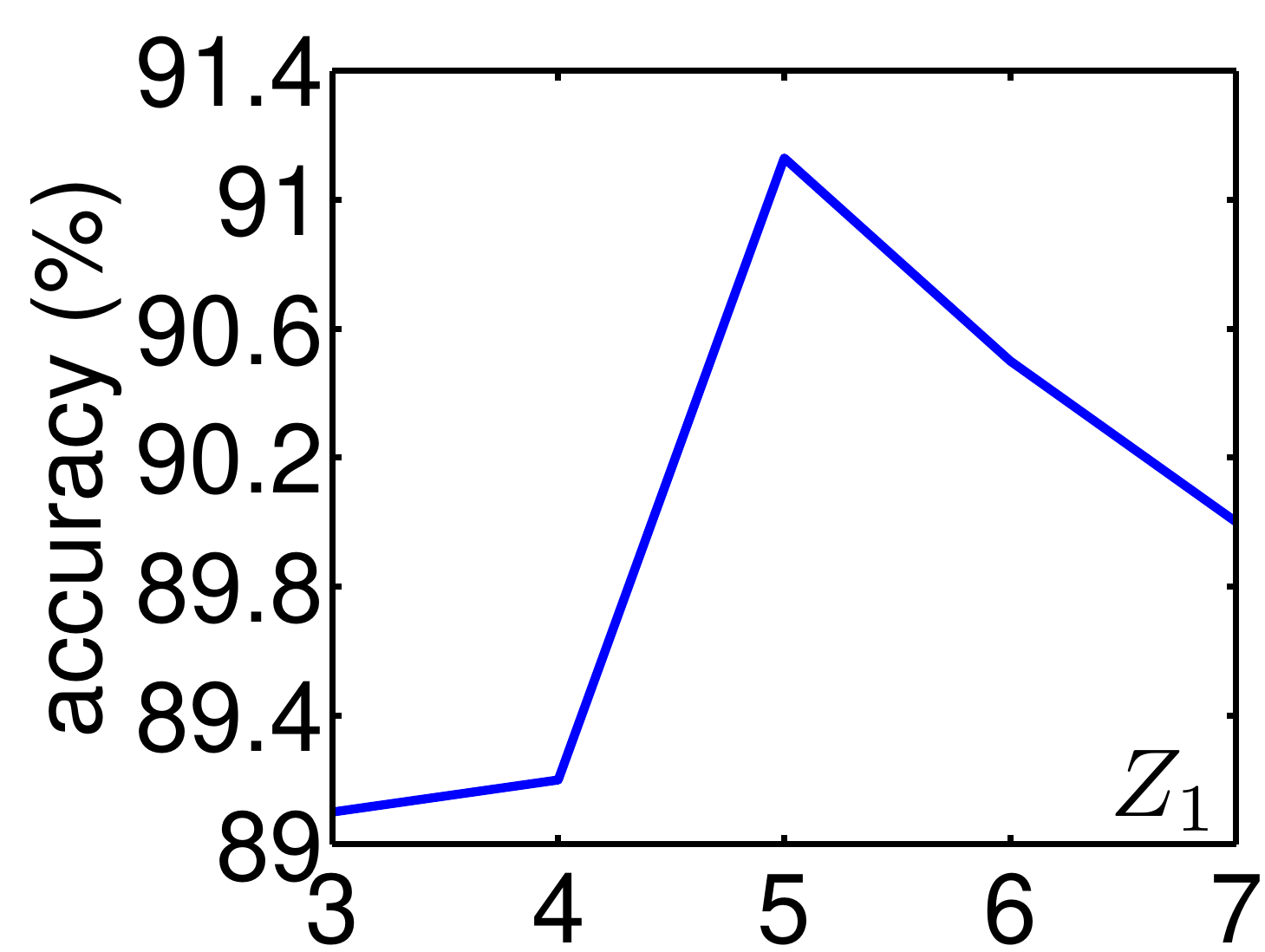}
\ifdefined\arxiv
\else
\vspace{-0.3cm}
\fi
\caption{\label{fig:z1a}}
\end{subfigure}
\begin{subfigure}[b]{0.42\linewidth}
\centering
\includegraphics[trim=0 0 0 0, clip=true, width=2.25cm, height=2cm]{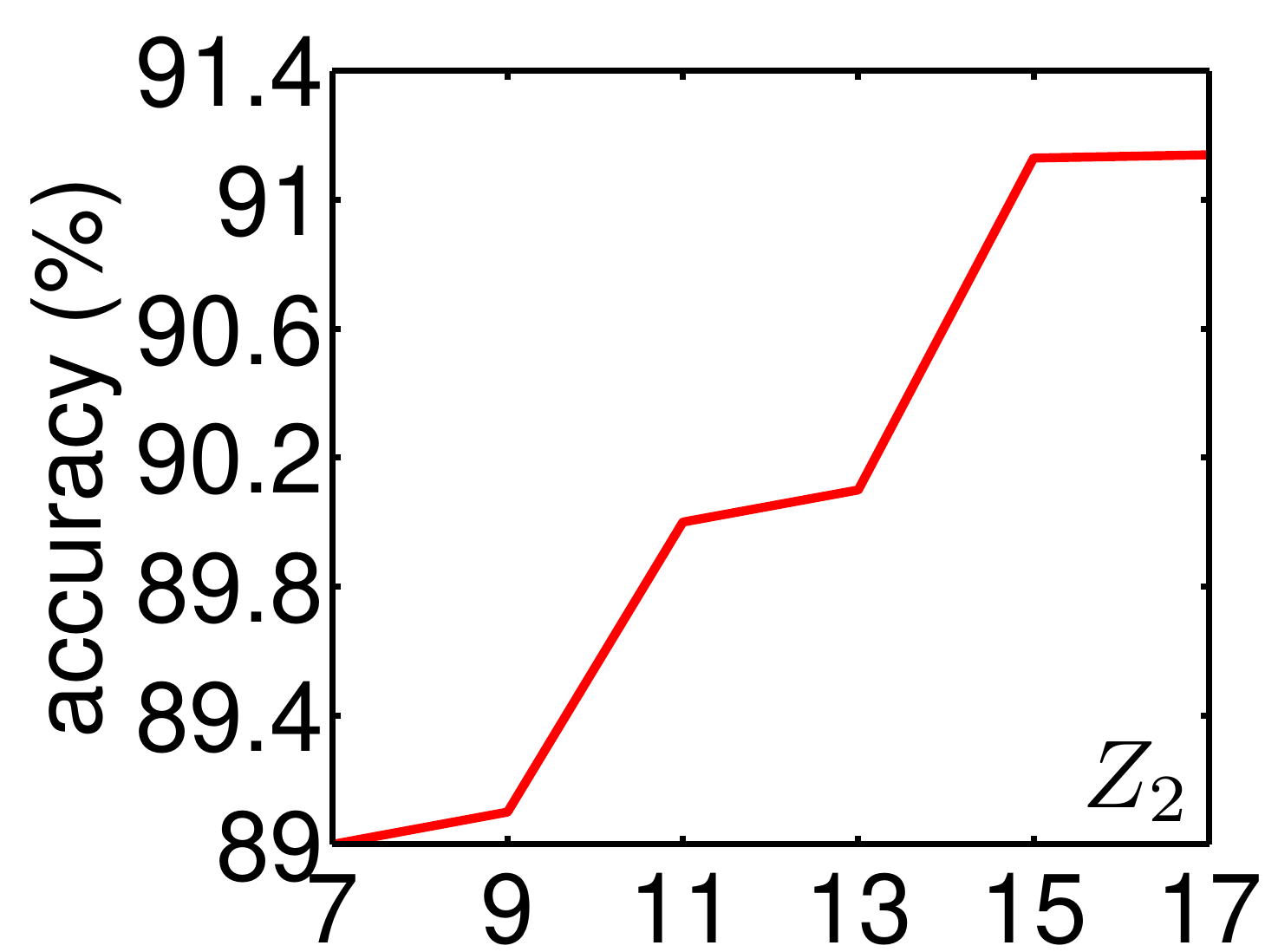}
\ifdefined\arxiv
\else
\vspace{-0.3cm}
\fi
\caption{\label{fig:z2b}}
\end{subfigure}
\ifdefined\arxiv
\vspace{-0.1cm}
\else
\vspace{0.1cm}
\fi
\captionof{figure}{Accuracy w.r.t. $Z_1$ ($Z_2\!=\!15$) and $Z_2$ ($Z_1\!=\!5$) on SBU in Figures \ref{fig:z1a} and \ref{fig:z2b}.
}
\label{fig:params_z}
%
%
%
}
\vspace{-0.5cm}
\end{table}
%
We avoid data augmentation or pre-processing. Protocol \cite{zhu2013fusing} has 2 splits: half of the subjects for training and half for testing. NTU training set is our source.


\vspace{0.1cm}
\noindent{\bf{Experiments.}} 
%
%
Below, we focus on the following types of experiments, each utilizing our encoder which transforms sequences of 3D body skeleton joints into feature maps:

\begin{enumerate}
\item Target-only: only target dataset is used for training and testing (no domain adaptation). 
\item Source+target: the source and target datasets (both training and validation splits) are combined into one larger dataset. Testing is performed on the target testing set only. No domain adaptation is used but the network is trained on both domains. 
\item Second-order alignment: our extended So-HoT model applies the domain adaptation between the source and target training datapoints. We perform the alignment of second-order statistics whenever the source and target class names match.
\end{enumerate}

\begin{table}[b]
\vspace{-0.3cm}
\hspace{-0.1cm}
	\parbox{0.5\textwidth}{
	\renewcommand{\arraystretch}{0.8}{
	\setlength{\tabcolsep}{0.25em}
	\fontsize{9}{9.5}\selectfont
  \begin{tabular}{c|c}
		\hline
    Methods&Accuracy\\
		\hline
     Raw Skeleton \cite{yun2012two}& $49.7\%$\\
     Hierarchical RNN \cite{du2015hierarchical}& $80.35\%$\\
     Deep LSTM \cite{zhu2016co} & $86.03\%$\\
     Deep LSTM $+$ Co-occurrence \cite{zhu2016co}& $90.41\%$\\
     ST-LSTM \cite{liu2016spatio}&$88.6\%$\\
     ST-LSTM $+$ Trust Gate \cite{liu2016spatio}&$93.3\%$\\
     Frames $+$ CNN \cite{new_skeleton}&$90.8\%$\\
     Clips $+$ CNN $+$ MTLN \cite{new_skeleton}&$93.57\%$\\\hline   
     SBU  only ({\em target})&91.13\%\\
     NTU+SBU combined ({\em source+target})&91.52\%\\
     Second-order alignment&\textbf{94.36\%}\\
	\hline
\end{tabular}
}
\vspace{0.2cm}
\caption{Results on the SBUKinect dataset.}\label{tab:sbu}
}
\parbox{0.5\textwidth}{
\renewcommand{\arraystretch}{0.8}{
\setlength{\tabcolsep}{0.25em}
\fontsize{9}{9.5}\selectfont
  \begin{tabular}{c|c}
		\hline
    Methods&Accuracy\\
		\hline
     3D Histogram (leave one out) \cite{xia2012view}& $90.92\%$\\
     Lie Group \cite{vemulapalli_SE3}& $97.08\%$\\
     SCK $+$ DCK \cite{tensor_eccv} & $98.39\%$\\
     Skeleton Joint Features \cite{zhu2013fusing}& $90.9\%$\\
     ST-LSTM $+$ Trust Gate \cite{liu2016spatio}&$95.0\%$\\
     Elastic Functional Coding \cite{anirudh2015elastic}&$94.9\%$\\\hline   
     UTK only ({\em target})&$96.5\%$\\
     NTU+UTK combined ({\em source+target})&$97.5\%$\\
     Second-order alignment &\textbf{98.9\%}\\
	\hline
\end{tabular}
}
\vspace{0.2cm}
  \caption{Results on the UTKinect dataset.}\label{tab:utk}
}
\vspace{-0.3cm}
\end{table}


\begin{table}
\parbox{0.62\textwidth}{
\renewcommand{\arraystretch}{0.8}{
\setlength{\tabcolsep}{0.25em}
\fontsize{9}{9.5}\selectfont
  \begin{tabular}{c|c}
		\hline
    Methods&Cross subject\\
		\hline
     Hierarchical RNN \cite{du2015hierarchical}& $59.1\%$\\
     Deep RNN \cite{shahroudy2016ntu}& $59.3\%$\\
     Deep LSTM \cite{shahroudy2016ntu}& $60.7\%$\\
     ST-LSTM $+$ Trust Gate \cite{liu2016spatio}&$69.2\%$\\
     Frames $+$ CNN \cite{new_skeleton}&${\bf 75.73}\%$\\\hline    
     NTU only ({\em target})&$74.52\%$\\
		 NTU+UTK+SBU combined ({\em source+target})&$74.65\%$\\
     Second-order alignment (UTK$\rightarrow$NTU)&$74.91\%$\\
     Second-order alignment (SBU$\rightarrow$NTU)&$74.83\%$\\
     Second-order alignment (UTK$+$SBU$\rightarrow$NTU)&${\bf 75.35\%}$\\
	\hline
\end{tabular}
}
\vspace{0.2cm}
\caption{Results on the NTU dataset.}\label{tab:ntu}
}
\vspace{0.1cm}
\parbox{0.375\textwidth}{
%
%
\vspace{-0.3cm}
\begin{subfigure}[b]{0.49\linewidth}
\centering
\includegraphics[trim=0 0 0 0, clip=true, width=2.25cm, height=2cm]{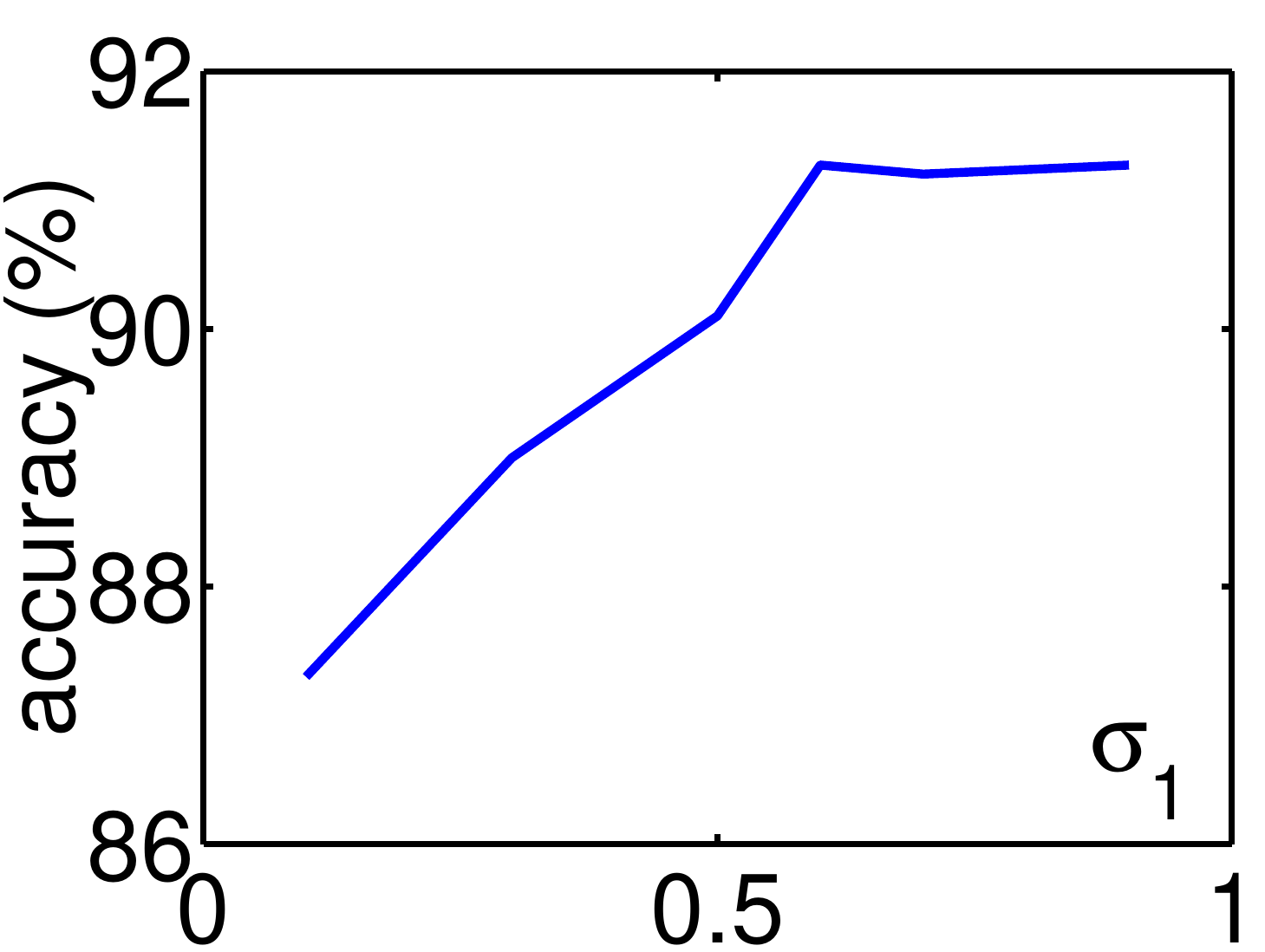}
\ifdefined\arxiv
\else
\vspace{-0.3cm}
\fi
\caption{\label{fig:sig1a}}
\end{subfigure}
\begin{subfigure}[b]{0.49\linewidth}
\centering
\includegraphics[trim=0 0 0 0, clip=true, width=2.25cm, height=2cm]{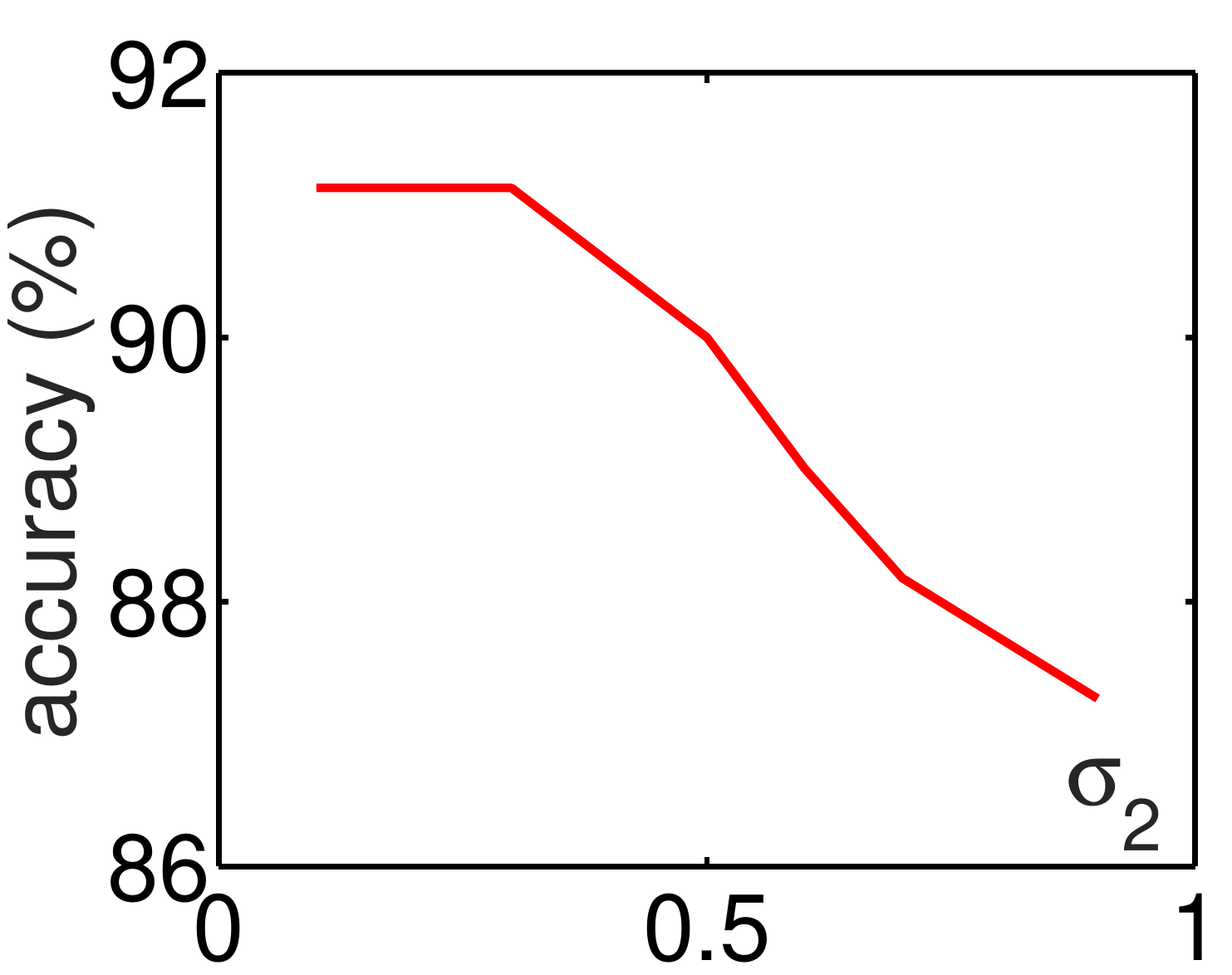}
\ifdefined\arxiv
\else
\vspace{-0.3cm}
\fi
\caption{\label{fig:sig2b}}
\end{subfigure}
\ifdefined\arxiv
\vspace{-0.45cm}
\else
\vspace{-0.25cm}
\fi
\captionof{figure}{Sensitivity w.r.t. param. $\sigma_1$ and $\sigma_2$ on SBU. Figures \ref{fig:sig1a} and \ref{fig:sig2b} show the accuracy w.r.t. $\sigma_1$ ($\sigma_2\!=\!0.3$) and $\sigma_2$ ($\sigma_1\!=\!0.6$), resp.
}\vspace{-0.3cm}
\label{fig:params}
%
%
%
}
\vspace{-0.5cm}
\end{table}

\vspace{-0.2cm}
\noindent{\bf{No Domain Adaptation.}} 
Firstly, we compare our encoding to texture-based representation \cite{new_skeleton}. Approach \cite{new_skeleton} forms 4 arrays of cylindrical coordinates of 3D skeleton body joints, each translated w.r.t. each 4 pre-defined key-joints. Such arrays are later resized, cropped \etc. and fed to network via multiple CNN inputs. They require a dedicated CNN pipeline which combines all these arrays. To make a fairer comparison to our encoding and use an off-the-shelf CNN setting, we simplified  representation \cite{new_skeleton} to use only a single body key-joint center for translation. We use the same setting for our encoding and \cite{new_skeleton} based on ResNet-50. We do not use a domain adaptation for results in Table \ref{tab:text}. We include however a variant of method \cite{new_skeleton} which generates 3 texture images (one per each cylindrical coordinate). Thus, these 3 texture images are passed via 3 CNN streams and their {\em FC} vectors are concatenated. 

Table \ref{tab:text} shows the comparison of our texture-like feature map encoding against method \cite{new_skeleton}. With $3\!\times$ more texture images taking $3\!\times$ more time to process via 3 CNN streams, method ({\em Cylindrical textures, $3\!\times$CNN}) \cite{new_skeleton} performs $\sim$0.7--0.9\% worse than ours. Moreover, for fairness, we next combine their 3 texture images (one per each cylindrical coordinate) into an RGB-like texture and passed via 1 CNN stream ({\em Cylindrical textures, $1\!\times$CNN}). Table \ref{tab:text} shows that given the same ResNet-50 pipeline, our method outperforms theirs by $\sim$1.8\% and 1.4\% on SBU and UTK. Figures \ref{fig:params_z} and \ref{fig:params} show that our encoder is not too sensitive w.r.t. the choice of $Z_1$, $Z_2$, $\sigma_1$ and $\sigma_2$ on the SBU dataset (no domain adaptation). Figure {\color{cyan}1} (supplementary material) shows a similar analysis on the UTK dataset.

Although idea \cite{new_skeleton} appears somewhat related to ours, the inner workings of both methods differ \eg, our method is mathematically inspired to attain desired shift-invariance w.r.t. 3D positions of coordinates and the temporal domain. In contrast, approach \cite{new_skeleton} 
is hand-crafted.

\vspace{0.1cm}
\noindent{\bf{Domain Adaptation Setting.}} 
Having shown that our encoder 
outperforms \cite{new_skeleton} given the same pipeline, we discuss below results on the supervised domain adaptation pipeline.

In Table \ref{tab:sbu}, we compare our method against state-of-the-art results on the SBU dataset. After enabling the domain adaptation algorithm ({\em second-order alignment}), the accuracy increases by \textbf{3.23\%} over training on the target data only ({\em target}). Our method also outperforms naive training on the combined source and target data ({\em source+target}) by \textbf{1.84\%}. 
We note that without any data augmentation, our method outperforms more complicated approaches which utilize numerous texture-like representations per sequence combined with several CNN streams and a fusion network ({\em Clips+CNN+MTLN}) \cite{new_skeleton}. This shows the effectiveness of our supervised domain adaptation on sequences of 3D body skeleton joints. 

Table \ref{tab:utk} shows on the UTK dataset that domain adaptation  ({\em second-order alignment}) outperforms the baseline ({\em target}) and the naive fusion ({\em source+target}) by \textbf{2.4\%} and \textbf{1.4\%}. 

Table \ref{tab:ntu} presents the transfer results from UTK and/or SBU to NTU. Transferring the knowledge from small- to large-scale datasets is a difficult task. However, by combining UTK and SBU to form a source dataset, we were able to still gain 0.8\% improvement over the baseline ({\em target}). We obtain results similar to \cite{new_skeleton} with a much simpler pipeline.



\vspace{-0.25cm}
\section{Conclusions}

In this paper, we have demonstrated that sequences of 3D body skeleton joints can be easily encoded with the use of appropriately designed kernel function. A linearization of such a kernel function produces texture-like feature maps which constitute a first feed-forward layer further interconnected with off-the-shelf CNNs. Moreover, we have also demonstrated that the supervised domain adaptation can be performed on such representations and that small-scale Kinect-based datasets can benefit from the knowledge transfer from the large-scale NTU dataset. We believe our contributions lead to state-of-the-art results. They also open up interesting avenues on how to use the time sequences with traditional off-the-shelf CNNs and how to leverage the abundance of the skeleton-based action recognition datasets.

\begin{appendices}

\section{Supervised Domain Adaptation \cite{me_domain}}
\label{sec:dom_adapt}

For the full details of the {\em So-HoT} algorithm, please refer to paper \cite{me_domain}. Below, we review the core part of their algorithm for the reader's convenience.
Suppose $\idx{N}$ and $\idx{N^*}\!$ are the indexes of $N$ source and $N^*\!$ target training data points. $\idx{N_c}$ and $\idx{N_c^*}\!$ are the class-specific indexes for $c\!\in\!\idx{C}$, where $C$ is the number of classes. Furthermore, suppose 
we have feature vectors from an {\em FC} layer 
of the source network stream, one per an action sequence or image, and their associated labels. Such pairs are given by $\mLam\!\equiv\!\{(\vphi_n, y_n)\}_{n\in\idx{N}}$, where $\vphi_n\!\in\!\mbr{d}$ and $y_n\!\in\!\idx{C}$, $\forall n\!\in\!\idx{N}$. 
For the target data, by analogy, we define pairs $\mLam^{*\!}\!\equiv\!\{(\vphi^*_n, y^*_n)\}_{n\in\idx{N}^*}$, where $\vphi^*\!\!\in\!\mbr{d}$ and $y^*_n\!\!\in\!\idx{C}$, $\forall n\!\in\!\idx{N}^*$. 
Class-specific sets of feature vectors are given as $\mPhi_c\!\equiv\!\{\vphi^c_n\}_{n\in\idx{N_c}}$ and $\mPhi_c^*\!\!\equiv\!\{\vphi^{*c}_n\}_{n\in\idx{N_c^*\!}}$, $\forall c\!\in\!\idx{C}$. Then $\mPhi\!\equiv\!(\mPhi_1,\cdots,\mPhi_C)$ and $\mPhi^*\!\!\equiv\!(\mPhi^*_1,\cdots,\mPhi^*_C)$.
 The asterisk in superscript (\eg~${\vphi}^*$) denotes variables related to the target network while the source-related variables have no asterisk. 
\ifdefined\arxiv
Figure \ref{fig:our_pipe}
\else
Figure {\color{cyan}4} (the main submission) 
\fi
shows the setup we use. The {\em So-HoT} problem is posed as a trade-off between the classifier and alignment losses $\ell$ and $\hbar$:
\vspace{0.00cm}
\begin{align}
\hspace{0.2cm}
&\!\!\!\!\!\!\!\!\!\!\!\!\!\!\!\!\!\!\argmin\limits_{\;\;\substack{\mW\!,\mW^*\!\!\!\!,\mP,\mP^*\!\!\comment{ ,\vec{\zeta},\vec{\overbartwo{\zeta}} }\\\;\;\;\;\text{s. t. }||\vphi_n||_2^2\leq\tau,\\\;\;\;\;\;\;\;\;\,||\vphi^*_{n'}||_2^2\leq\tau,\\\;\;\;\;\forall n\in\idx{N}\!, n'\!\in\idx{N}^*}} 
\ell\!\left(\mW\!,\mLam\right)\!+\!\ell\!\left(\mW^{*\!}\!,\mLam^{*\!}\right)\!+\!\eta||\mW\!-\!\mW^{*\!}||_F^2\;+\label{eq:main_obj1}\\[-35pt]
%
&\qquad\qquad\quad\!\!\!\!\!\underbrace{\frac{\alpha_1}{C}\!\!\sum_{c\in\idx{C}}\!\comment{ \zeta_{c} }||\cov_c\!-\!\cov^*_c||_F^2
\!+\!\!\frac{\alpha_2}{C}\!\!\sum_{c\in\idx{C}}\!\comment{ \overbartwo{\zeta}_c}||\vmu_c\!\!-\!\!\vmu_c^*||_2^2.
}_{\hbar(\mPhi,\mPhi^*\!)}\nonumber
%
%
\vspace{0.2cm}
\end{align}
%
For $\ell$,  a generic Softmax loss is employed. For the source and target streams, the matrices $\mW,\mW^*\!\!\in\!\mbr{d\times C}$ contain unnormalized probabilities. 
In Equation \eqref{eq:main_obj1}, separating the class-specific distributions is addressed by $\ell$ while attracting the within-class scatters of both network streams is handled by $\hbar$. Variable 
$\eta$ controls the proximity between $\mW$ and $\mW^*\!$ which encourages the similarity between decision boundaries of classifiers.

The loss $\hbar$ depends on two sets of variables $(\mPhi_1,\cdots,\mPhi_C)$ and $(\mPhi^*_1,\cdots,\mPhi^*_C)$ -- one set per network stream. Feature vectors $\mPhi(\mP)$ and $\mPhi^*\!(\mP^*\!)$ depend on the parameters of the source and target network streams $\mP$ and $\mP^*\!$ that we optimize over.
%
$\cov_c\!\equiv\!\cov(\mPhi_c)$, $\cov^*_c\!\equiv\!\cov(\mPhi^*_c)$, $\vmu_c(\mPhi)$ and $\vmu^*_c(\mPhi^*)$ denote the covariances and means, respectively, one covariance/mean pair per network stream per class. Coefficients $\alpha_1$, $\alpha_2$ control the degree of the scatter and mean alignment, $\tau$ controls the $\ell_2$-norm of feature vectors.

\section{Modifications to the So-HoT Approach}
\label{sec:modif}

Algorithm \ref{alg:alg1} details how we perform domain adaptation. We enable the alignment loss $\hbar$ only if the source and target batches correspond to the same class. Otherwise, the alignment loss is disabled and the total loss uses only the classification log-losses $\ell_{src}$ and $\ell_{trg}$. To generate the source and target batches that match w.r.t. the class label, we re-order source and target datasets class-by-class and thus each source/target batch contains only one class label at a time. Once all source and target datapoints with matching class labels are processed, remaining datapoints are processed next. Lastly, we refer readers interested in the details of the So-HoT algorithm and loss $\hbar$ to paper \cite{me_domain}.

\begin{algorithm}[h]
\vspace{-0.05cm}
\caption{Batch generation + a single epoch of the training procedure on the source and target datasets.}\label{alg:alg1}
\begin{algorithmic}[1]
\State $src\_data:=\text{sort\_by\_class\_label}(src\_data)$
\State $target\_data:=\text{sort\_by\_class\_label}(target\_data)$
\State $C_s  $\Comment{Number of the source classes}
\State $C_t  $\Comment{Number of the target classes}
\State $C_{s\cap t} $\Comment{Number of classes in common}
\Procedure{epoch}{$src\_data,target\_data,batch\_size$}
\Comment{Training (one epoch)}
\For{$i \gets 1:max(C_s, C_t)$ }
\If{$i \leq C_s $} \State $batch_s \gets \text{Choose}(src\_data, i, batch\_size)$
\Comment{`Choose' pre-fetches data of class i}
\Else \State $batch_s \gets \text{Choose}(src\_data, rnd(), batch\_size)$
\Comment{`Choose' pre-fetches data of random class}
\EndIf
\If{$i \leq C_t $} \State $batch_t \gets \text{Choose}(target\_data, i, batch\_size)$ 
\Else \State $batch_t \gets \text{Choose}(target\_data, rnd(), batch\_size)$
\EndIf
\If{$i \leq C_{s\cap t} $} 
\State $Loss \gets \ell_{src} + \ell_{trg} + \hbar$
\Else 
\State $Loss \gets \ell_{src} + \ell_{trg}$
\EndIf
\State Forward($net\_data$, $batch\_s$, $batch\_t$)
\State Backward($net\_data$, $batch\_s$, $batch\_t$)
\State Update($net\_data$, $batch\_s$, $batch\_t$)
\EndFor
\EndProcedure
\end{algorithmic}
\vspace{-0.15cm}
\end{algorithm}


%
\begin{figure*}[h]
\centering
\ifdefined\arxiv
\else
\vspace{-0.35cm}
\fi
\vspace{-0.3cm}
\begin{subfigure}[b]{0.245\linewidth}
\centering
\includegraphics[trim=0 0 0 0, clip=true, height=2.4cm]{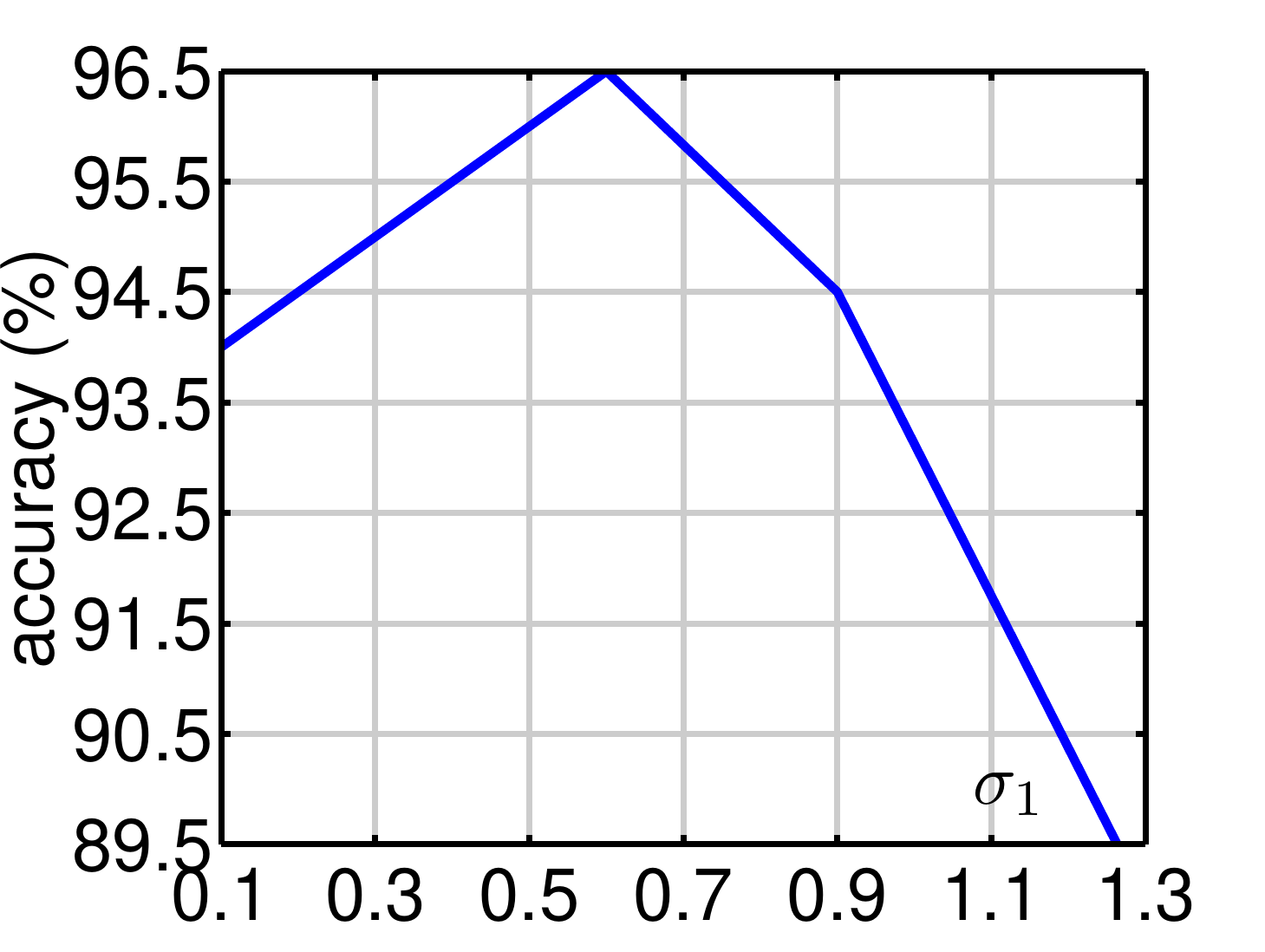}
\ifdefined\arxiv
\else
\vspace{-0.5cm}
\fi
\caption{\label{fig:utk_sig1a}}
\end{subfigure}
\begin{subfigure}[b]{0.245\linewidth}
\centering
\includegraphics[trim=0 0 0 0, clip=true, height=2.4cm]{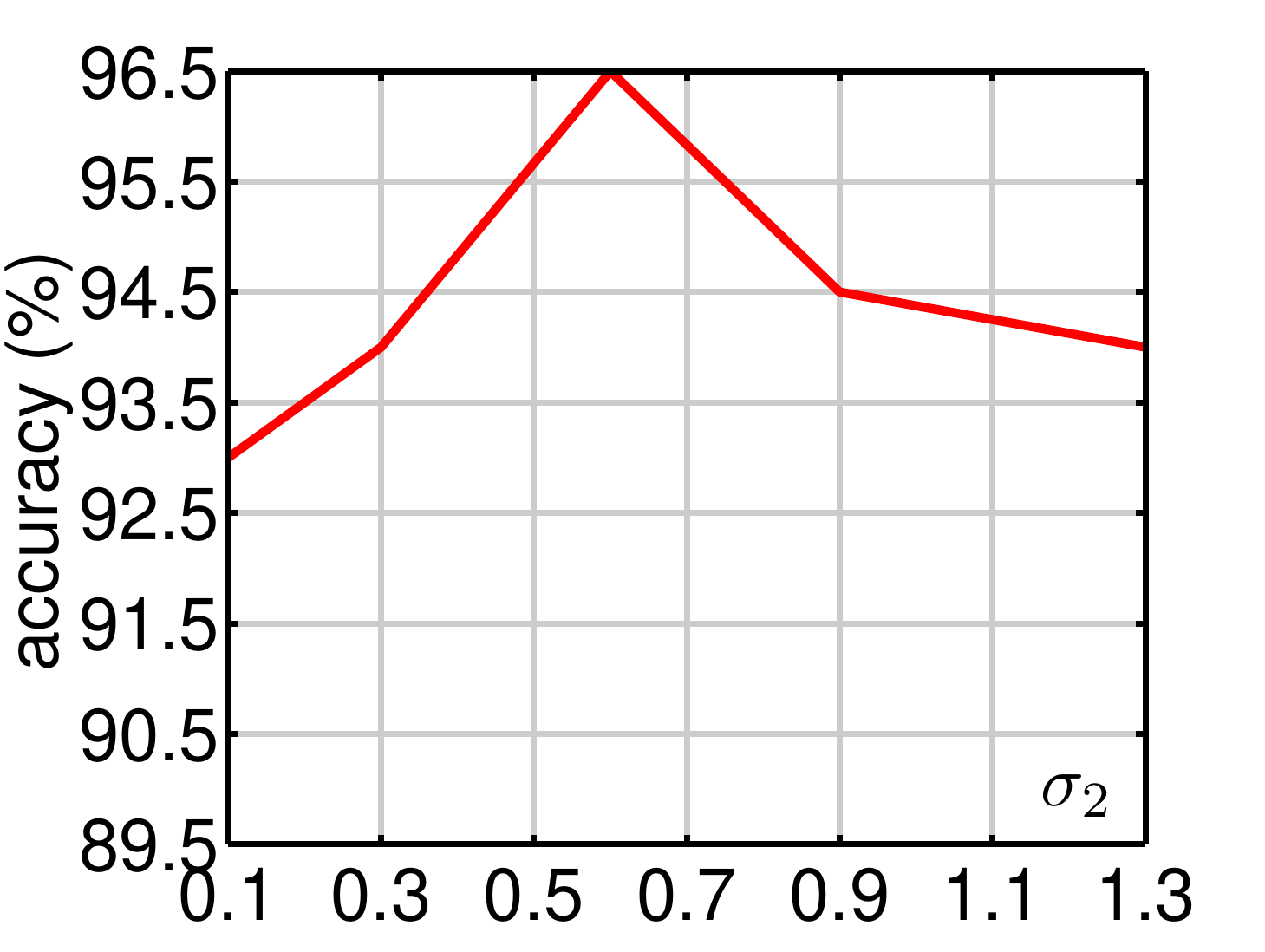}
\ifdefined\arxiv
\else
\vspace{-0.5cm}
\fi
\caption{\label{fig:utk_sig2b}}
\end{subfigure}
\begin{subfigure}[b]{0.245\linewidth}
\centering
\includegraphics[trim=0 0 0 0, clip=true, height=2.4cm]{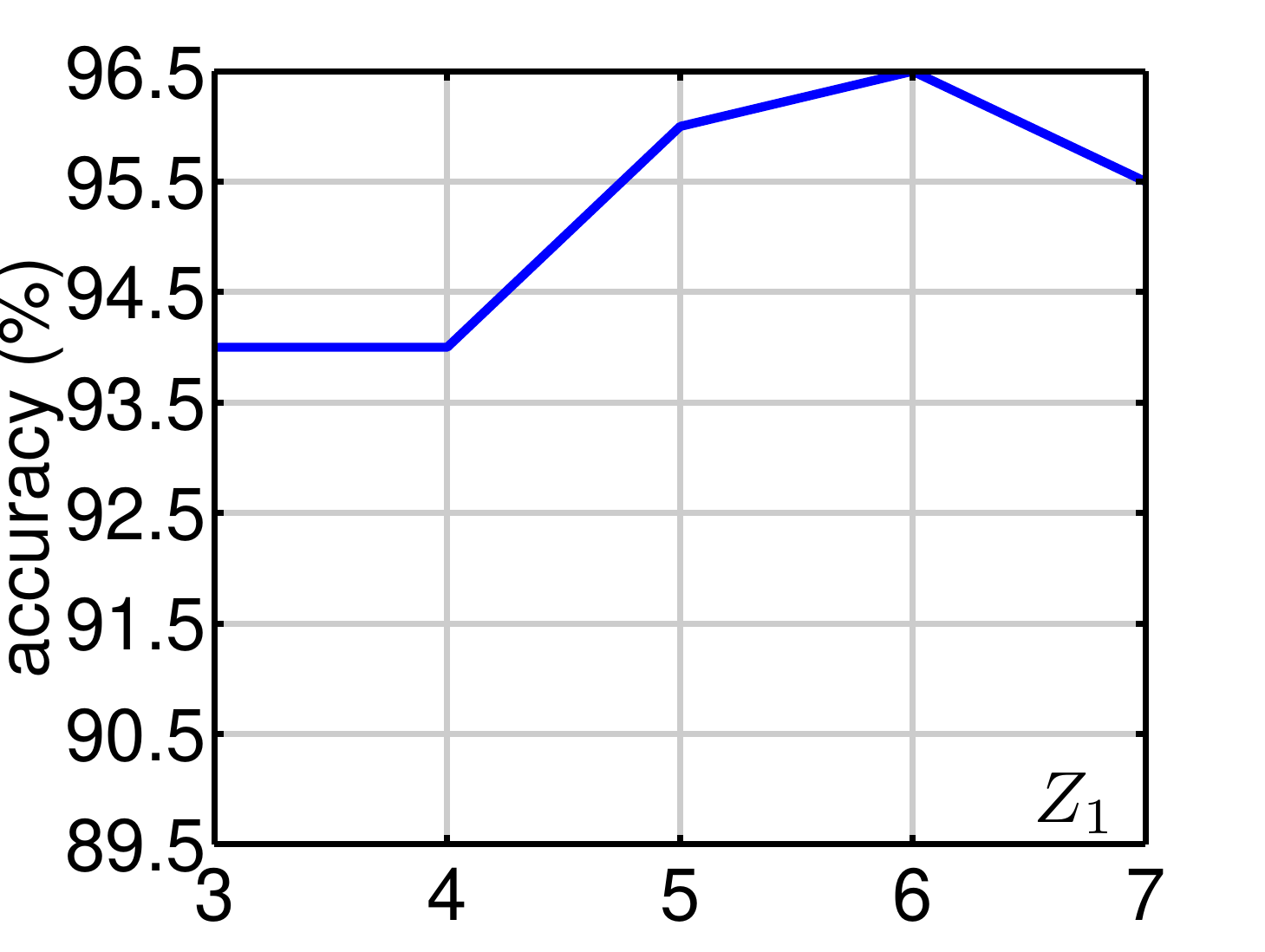}
\ifdefined\arxiv
\else
\vspace{-0.5cm}
\fi
\caption{\label{fig:utk_z1a}}
\end{subfigure}
\begin{subfigure}[b]{0.245\linewidth}
\centering
\includegraphics[trim=0 0 0 0, clip=true, height=2.4cm]{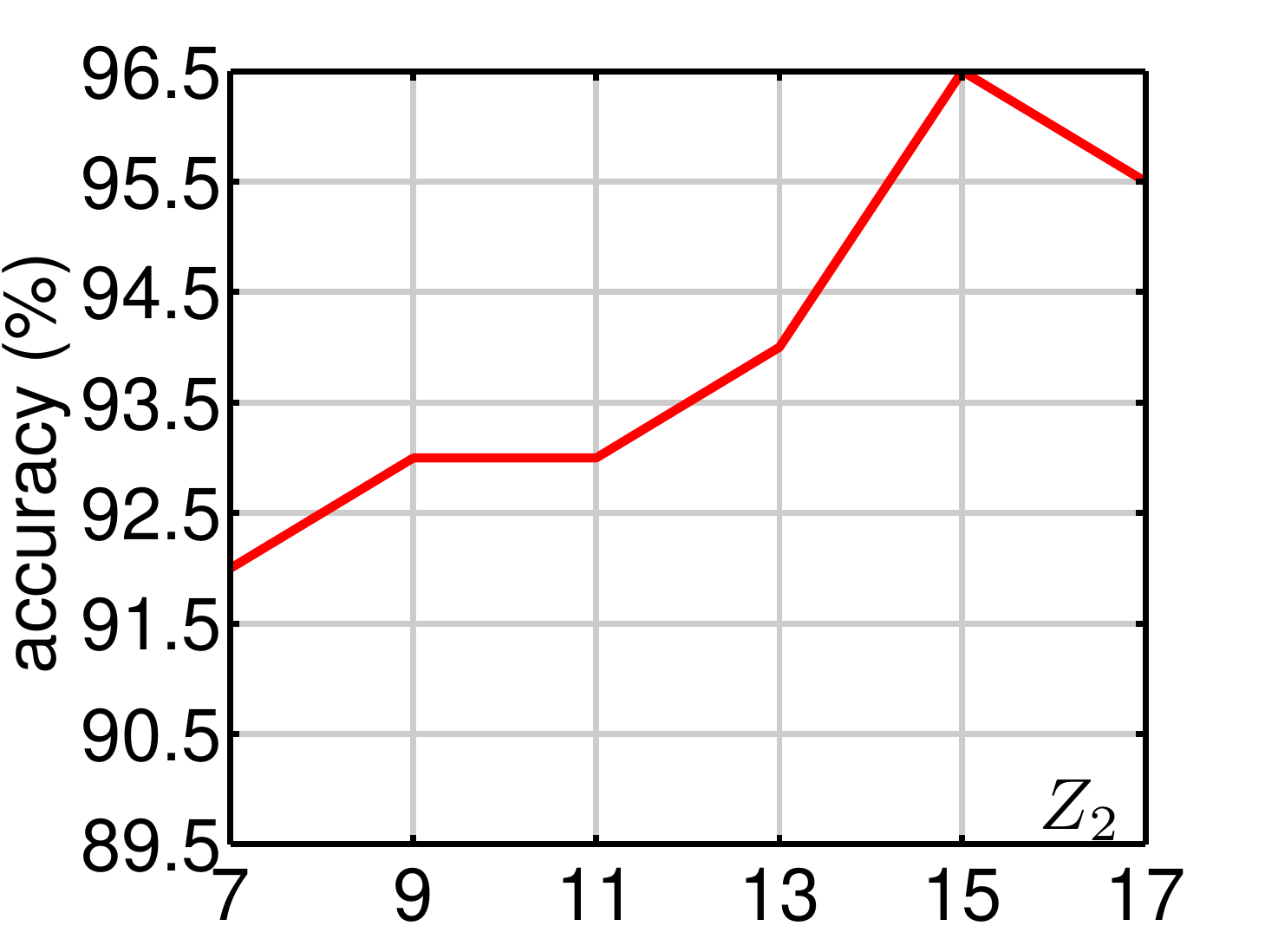}
\ifdefined\arxiv
\else
\vspace{-0.5cm}
\fi
\caption{\label{fig:utk_z2b}}
\end{subfigure}
\vspace{0.0cm}
\ifdefined\arxiv
\vspace{-0.5cm}
\else
\fi
\caption{Sensitivity w.r.t. parameters $\sigma_1$ and $\sigma_2$ on UTK. Figures \ref{fig:utk_sig1a}, \ref{fig:utk_sig2b}, \ref{fig:utk_z1a} and \ref{fig:utk_z2b} show the accuracy w.r.t. $\sigma_1$ ($\sigma_2\!=\!0.6$), $\sigma_2$ ($\sigma_1\!=\!0.6$), $Z_1$ ($Z_2\!=\!15$) and $Z_2$ ($Z_1\!=\!5$), respectively.
}\vspace{-0.3cm}
\label{fig:params}
\end{figure*}

\end{appendices}


{\small

}

\end{document}